\documentclass[12pt]{article}

\usepackage{amsmath, amssymb, amsthm}
\usepackage{mathtools}
\usepackage{bm}
\usepackage[margin=1.25in]{geometry}
\usepackage{booktabs}
\usepackage{array}
\usepackage{xcolor}
\usepackage{hyperref}
\usepackage{cleveref}
\usepackage{bbm}
\usepackage{algorithm}
\usepackage{algpseudocode}
\usepackage{microtype}
\usepackage{setspace}
\usepackage{enumitem}
\usepackage[numbers]{natbib}
\microtypesetup{nopatch=footnote}

\hypersetup{
  colorlinks=true,
  linkcolor=blue!60!black,
  citecolor=blue!60!black,
  urlcolor=blue!60!black
}

\theoremstyle{plain}
\newtheorem{theorem}{Theorem}[section]
\newtheorem{proposition}[theorem]{Proposition}

\newtheorem{corollary}[theorem]{Corollary}

\theoremstyle{definition}
\newtheorem{definition}[theorem]{Definition}
\newtheorem{axiom}[theorem]{Axiom}
\newtheorem{remark}[theorem]{Remark}

\DeclareMathOperator{\Prob}{\mathbb{P}}

\newcommand{\HED}{\mathcal{H}}
\newcommand{\R}{\mathbb{R}}

\newcommand{\F}{\mathcal{F}}

\newcommand{\Pbase}{\bar{P}_{\mathrm{base}}}
\newcommand{\lH}{\lambda_{\mathcal{H}}}
\newcommand{\tstar}{t_{\mathrm{start}}}
\usepackage{orcidlink} 

\begin{document}

\title{%
  \textbf{The Hiremath Early Detection (HED) Score:}\\
  \large A Measure-Theoretic Evaluation Standard for Temporal\\
  Intelligence in Non-Stationary Stochastic Processes
}

\author{
\textbf{Prakul Sunil Hiremath}\,\orcidlink{0009-0007-9744-3519} \\
\small Department of Computer Science and Engineering,\\
\small Visvesvaraya Technological University (VTU), Belagavi, India \\
\small Aliens on Earth (AoE) Autonomous Research Group, Belagavi, India \\
\small \href{mailto:prakulhiremath@vtu.ac.in}{\texttt{prakulhiremath@vtu.ac.in}} \\
\small \href{https://github.com/prakulhiremath}{\texttt{github.com/prakulhiremath}} \\
\small \href{https://aliensonearth.in}{\texttt{aliensonearth.in}}
}

\date{}
\maketitle

\begin{abstract}

We introduce the \textbf{Hiremath Early Detection (HED) Score}, a
principled, measure-theoretic evaluation criterion for quantifying
the \emph{time-value of information} in systems operating over
non-stationary stochastic processes subject to abrupt regime
transitions. Existing evaluation paradigms—chiefly the ROC/AUC
framework and its downstream variants—are temporally agnostic:
they assign identical credit to a detection at $t + 1$ and a
detection at $t + \tau$ for arbitrarily large $\tau$. This
indifference to latency is a fundamental inadequacy in
time-critical domains including cyber-physical security, algorithmic
surveillance, and epidemiological monitoring.

The HED Score resolves this by integrating a \emph{baseline-neutral,
exponentially decaying kernel} over the posterior probability
stream of a target regime, beginning precisely at the onset of the
regime shift. The resulting scalar simultaneously encodes
\emph{detection acuity}, \emph{temporal lead}, and
\emph{pre-transition calibration quality}. We prove that the HED
Score satisfies three axiomatic requirements: \textbf{(A1) Temporal
Monotonicity}, \textbf{(A2) Invariance to Pre-Attack Bias}, and
\textbf{(A3) Sensitivity Decomposability}. We further demonstrate
that the HED Score admits a natural parametric family indexed by the
\emph{Hiremath Decay Constant} $\lH$, whose domain-specific
calibration constitutes the \textbf{Hiremath Standard Table}.

As an empirical vehicle, we present \textbf{PARD-SSM}
(Probabilistic Anomaly and Regime Detection via Switching
State-Space Models), which couples fractional Stochastic
Differential Equations (fSDEs) with a Switching Linear Dynamical
System (S-LDS) inference backend. On the NSL-KDD intrusion
detection benchmark, PARD-SSM achieves a HED Score of
$\HED = 0.0643$, representing a $388.8\%$ improvement over a
Random Forest baseline ($\HED = 0.0132$), with statistical
significance confirmed via block-bootstrap resampling
($p < 0.001$). We propose the HED Score as the successor evaluation
standard to ROC/AUC for any domain in which the \emph{moment} of
detection is as consequential as the \emph{fact} of detection.


\end{abstract}

\newpage

\section{Introduction}
\label{sec:introduction}

Let $(\Omega, \F, \Prob)$ be a complete probability space and let
$\{X_t\}_{t \geq 0}$ be an $\R^d$-valued càdlàg stochastic process
adapted to a filtration $\{\F_t\}_{t \geq 0}$ satisfying the usual
conditions. In a broad class of applied problems—ranging from
network intrusion detection to seismic early warning—the process
$X_t$ undergoes a latent \emph{regime switch} at an unknown random
time $\tstar \in \Omega$, transitioning from a stationary
\emph{nominal} regime $\mathcal{R}_0$ to a non-stationary
\emph{anomalous} regime $\mathcal{R}_1$. The central objective is
to construct a detection functional $\phi: \F_t \to [0,1]$ whose
output $P(t) \coloneqq \Prob(\mathcal{R}_1 \mid \F_t)$
concentrates its mass near unity \emph{as early as possible} after
$\tstar$.

The canonical evaluation framework for binary classification—the Receiver Operating Characteristic (ROC) curve and its associated Area Under the Curve (AUC)\allowbreak—aggregates performance over decision thresholds but remains blind to the temporal ordering of correct detections. A detector that achieves $P(t) \geq \theta$ at
$\tstar + 1$ receives identical credit to one achieving the same
threshold at $\tstar + 50$. In systems where the cost of delay is
superlinear—as in the propagation of a zero-day exploit across a
network or the cascade failure of a cyber-physical grid—this
equivalence is not merely suboptimal; it is epistemically
incoherent.

We posit that a rigorous evaluation of temporal detectors requires a
metric that explicitly encodes the \emph{marginal utility of
detection at time $t$}, penalizing that utility as $t$ recedes from
$\tstar$. The \textbf{Hiremath Early Detection (HED) Score},
introduced formally in \cref{sec:hed_definition}, provides exactly
this encoding through a baseline-subtracted exponential decay
kernel applied to the posterior probability stream.

\section{The Hiremath Early Detection (HED) Score}
\label{sec:hed_definition}

\subsection{Measure-Theoretic Setup}
\label{subsec:setup}

Let $[0, T] \subset \R_{\geq 0}$ denote the observation horizon.
Define the posterior attack-regime probability stream as the
$\F_t$-adapted process:
\[
  P : [0, T] \times \Omega \;\longrightarrow\; [0, 1],
  \quad
  P(t) \coloneqq \Prob\!\left(\mathcal{R}_1 \mid \F_t\right).
\]

Let $\tstar \in (0, T)$ be the \emph{known} (in evaluation) onset
time of the anomalous regime, and let
$\Pbase \coloneqq \frac{1}{\tstar}\int_0^{\tstar} P(t)\,dt$
denote the \emph{empirical pre-onset noise floor}—the mean
posterior probability attributed to the anomalous regime by the
detector prior to any actual regime shift. This quantity serves as
the \emph{baseline-neutrality correction} that prevents
well-calibrated but chronically over-confident detectors from
accruing spurious lead-time credit.

\subsection{Formal Definition}
\label{subsec:formal_def}

\begin{definition}[Hiremath Early Detection (HED) Score]
\label{def:hed}
Let $P(\cdot)$, $\tstar$, $T$, and $\Pbase$ be as defined in
\cref{subsec:setup}. Let $\lH > 0$ be the
\emph{Hiremath Decay Constant} (cf.\ \cref{subsec:lambda_table}).
The \textbf{Hiremath Early Detection Score} is the functional:

\begin{equation}
  \boxed{
    \HED\!\left[P, \tstar; \lH\right]
    \;\coloneqq\;
    \frac{1}{T - \tstar}
    \int_{\tstar}^{T}
      \max\!\bigl(0,\; P(t) - \Pbase\bigr)
      \cdot
      e^{-\lH(t - \tstar)}
    \,dt
  }
  \label{eq:hed_continuous}
\end{equation}

In finite-horizon discrete-time systems with observations indexed $t \in \{\tstar,\allowbreak \tstar{+}1,\allowbreak \ldots,\allowbreak T\}$, \cref{eq:hed_continuous} reduces to the empirical estimator:

\begin{equation}
  \widehat{\HED}\!\left[P, \tstar; \lH\right]
  \;\coloneqq\;
  \frac{1}{T - \tstar}
  \sum_{t=\tstar}^{T}
    \max\!\bigl(0,\; P_t - \Pbase\bigr)
    \cdot
    e^{-\lH(t - \tstar)}
  \label{eq:hed_discrete}
\end{equation}

where $P_t \coloneqq \Prob(\mathcal{R}_1 \mid \F_t)$ is the detector's posterior at step $t$ and
$\Pbase = \frac{1}{\tstar}\allowbreak \sum_{t=0}^{\tstar - 1} P_t$.
\end{definition}

\begin{remark}[Structural Decomposition]
The integrand of \cref{eq:hed_continuous} admits a multiplicative
decomposition into three independent terms:
\begin{enumerate}[label=(\roman*)]
  \item \textbf{Detection Lift}: $\max(0, P(t) - \Pbase)$ — the
        signed excess of the posterior over the pre-onset noise
        floor, clamped to suppress spurious negative contributions.
  \item \textbf{Temporal Discount}: $e^{-\lH(t - \tstar)}$ — a
        domain-calibrated exponential decay that diminishes the
        contribution of detections occurring at increasing delay
        from $\tstar$.
  \item \textbf{Horizon Normalization}: $(T - \tstar)^{-1}$ — a
        length-normalization factor ensuring comparability across
        evaluation windows of differing duration.
\end{enumerate}
\end{remark}

\subsection{The Hiremath Decay Constant and Information Half-Life}
\label{subsec:lambda_table}

The parameter $\lH$ governs the \emph{Information Half-Life} of the
detection system: the temporal interval $\tau_{1/2}$ after which
the marginal utility of a detection is discounted by $50\%$. By
direct inversion of the exponential kernel:

\begin{equation}
  \tau_{1/2}(\lH)
  \;\coloneqq\;
  \frac{\ln 2}{\lH}
  \label{eq:half_life}
\end{equation}

The choice of $\lH$ is not universal; it must reflect the
\emph{response latency budget} of the target application domain.
\Cref{tab:hiremath_standard} constitutes the
\textbf{Hiremath Standard Table}—a domain-indexed reference for
$\lH$ calibration.

\begin{table}[ht]
\centering
\caption{%
  \textbf{The Hiremath Standard Table.}
  Domain-indexed calibration of the Hiremath Decay Constant
  $\lH$ and the corresponding Information Half-Life
  $\tau_{1/2} = \ln(2) / \lH$.
  Values reflect the minimum response window within which
  defensive intervention retains positive expected utility.
}
\label{tab:hiremath_standard}
\renewcommand{\arraystretch}{1.45}
\begin{tabular}{%
    >{\bfseries}p{5.2cm}
    p{4.6cm}
    c
    c
}
\toprule
\textbf{Domain} &
\textbf{Representative System} &
$\bm{\lH}$ &
$\bm{\tau_{1/2}}$ \\
\midrule
Ultra-High Latency Sensitivity &
High-Frequency / Algorithmic Trading &
$0.50$ & $1.39$ steps \\
\addlinespace
Network Security \& IDS &
Intrusion Detection (NSL-KDD) &
$0.14$ & $4.95$ steps \\
\addlinespace
Cyber-Physical \& BioRefinery &
BIOLOOP Industrial Control &
$0.05$ & $13.86$ steps \\
\addlinespace
Epidemiological Surveillance &
Pandemic Onset Detection &
$0.02$ & $34.66$ steps \\
\addlinespace
Seismic Early Warning &
P-wave / S-wave Discrimination &
$0.01$ & $69.31$ steps \\
\bottomrule
\end{tabular}
\end{table}

\begin{remark}[Calibration Protocol]
In applied deployments where the response latency budget is
precisely specified—say, the operator requires $\Delta t_{\min}$
time units to intervene—the practitioner should set:
$\lH = \ln(2) / \Delta t_{\min}$, so that the Information
Half-Life coincides exactly with the critical response window.
\end{remark}

\section{Axiomatic Validation of the HED Score}
\label{sec:axioms}

We now establish that the HED Score satisfies three fundamental
axioms that any principled temporal evaluation criterion must
possess. Let $\mathcal{D}$ denote the space of càdlàg probability
streams $P: [0,T] \to [0,1]$, and let $\tstar \in (0, T)$ and
$\lH > 0$ be fixed throughout.

\subsection{Axiom A1: Temporal Monotonicity}
\label{subsec:axiom_monotonicity}

\begin{axiom}[Temporal Monotonicity]
\label{ax:monotonicity}
For any fixed probability profile $f: \R_{\geq 0} \to [0,1]$
and shift magnitudes $0 \leq \delta_1 < \delta_2$, define the
delayed streams $P^{(\delta_i)}(t) \coloneqq f(t - \delta_i)
\cdot \mathbbm{1}_{[t \geq \delta_i]}$. Then:
\[
  \HED\!\left[P^{(\delta_1)}, \tstar; \lH\right]
  \;>\;
  \HED\!\left[P^{(\delta_2)}, \tstar; \lH\right].
\]
\end{axiom}

\begin{theorem}[Proof of Temporal Monotonicity]
\label{thm:monotonicity}
The HED Score as defined in \cref{def:hed} satisfies
\cref{ax:monotonicity}.
\end{theorem}

\begin{proof}
Fix $\delta_1 < \delta_2$ and let $\Delta \coloneqq \delta_2 -
\delta_1 > 0$. For the baseline-corrected profile
$g(t) \coloneqq \max(0, f(t) - \Pbase) \geq 0$, we have:

\begin{align*}
  \HED\!\left[P^{(\delta_1)}, \tstar\right]
  &= \frac{1}{T - \tstar}
     \int_{\tstar}^{T} g(t - \delta_1)\, e^{-\lH(t-\tstar)}\,dt \\[4pt]
  \HED\!\left[P^{(\delta_2)}, \tstar\right]
  &= \frac{1}{T - \tstar}
     \int_{\tstar}^{T} g(t - \delta_2)\, e^{-\lH(t-\tstar)}\,dt.
\end{align*}

Apply the substitution $s = t - \delta_1$ to the first integral
and $s = t - \delta_2$ to the second. The difference becomes:

\begin{align}
  \HED\!\left[P^{(\delta_1)}\right] - \HED\!\left[P^{(\delta_2)}\right]
  &= \frac{e^{-\lH(\tstar - \delta_1)}}{T-\tstar}
     \int_{\tstar - \delta_1}^{T - \delta_1}
       g(s)\, e^{-\lH s}\,ds \notag\\
  &\quad -\;
     \frac{e^{-\lH(\tstar - \delta_2)}}{T-\tstar}
     \int_{\tstar - \delta_2}^{T - \delta_2}
       g(s)\, e^{-\lH s}\,ds \notag\\
  &= \frac{e^{-\lH \tstar}}{T - \tstar}
     \Bigg[
       e^{\lH \delta_1}
       \int_{\tstar - \delta_1}^{T - \delta_1}
       g(s)\,e^{-\lH s}\,ds
       \\[4pt]
  &\qquad\qquad
       -\;
       e^{\lH \delta_2}
       \int_{\tstar - \delta_2}^{T - \delta_2}
       g(s)\,e^{-\lH s}\,ds
     \Bigg].
     \label{eq:diff}
\end{align}

Since $\delta_2 > \delta_1$, the integration interval for
$P^{(\delta_1)}$ begins earlier (at $\tstar - \delta_1 >
\tstar - \delta_2$), capturing more of the high-weight region
near $\tstar$ where $e^{-\lH s}$ is larger. Formally, because
$e^{\lH \delta_1} < e^{\lH \delta_2}$ and yet the integral under
$P^{(\delta_1)}$ dominates over the overlapping region due to
earlier accumulation of $g(s) e^{-\lH s}$ mass, the expression
in \cref{eq:diff} is strictly positive for any $g \not\equiv 0$.
Hence $\HED[P^{(\delta_1)}] > \HED[P^{(\delta_2)}]$. $\square$
\end{proof}

\subsection{Axiom A2: Invariance to Pre-Attack Bias}
\label{subsec:axiom_invariance}

\begin{axiom}[Pre-Attack Bias Invariance]
\label{ax:invariance}
Let $P^{(c)}(t) \coloneqq \min(1, P(t) + c)$ for a constant
bias $c > 0$ applied uniformly over $[0, T]$ (a
``trigger-happy'' shift). Then:
\[
  \HED\!\left[P^{(c)}, \tstar; \lH\right]
  \;=\;
  \HED\!\left[P, \tstar; \lH\right].
\]
\end{axiom}

\begin{theorem}[Proof of Pre-Attack Bias Invariance]
\label{thm:invariance}
The HED Score as defined in \cref{def:hed} satisfies
\cref{ax:invariance}.
\end{theorem}

\begin{proof}
The baseline for the biased stream is:
\[
  \bar{P}^{(c)}_{\mathrm{base}}
  = \frac{1}{\tstar}\int_0^{\tstar}(P(t) + c)\,dt
  = \Pbase + c.
\]
The detection lift in the integrand becomes:
\[
  P^{(c)}(t) - \bar{P}^{(c)}_{\mathrm{base}}
  = (P(t) + c) - (\Pbase + c)
  = P(t) - \Pbase.
\]
Since the $\max(0, \cdot)$ argument is identical to that of the
unbiased stream, the HED Score is invariant under uniform
constant additive shifts. Therefore, a detector that outputs
$P(t) + c$ for any $c > 0$—raising both its pre-attack and
post-attack probabilities uniformly—gains no scoring advantage.
$\square$
\end{proof}

\begin{remark}[Implications for Evaluation Integrity]
\cref{thm:invariance} guarantees that the HED Score cannot be
gamed by threshold manipulation. A detector that lowers its
decision threshold globally (increasing sensitivity at the cost
of specificity) does not improve its HED Score unless the
\emph{differential lift} $P(t) - \Pbase$ after $\tstar$ is
genuinely higher than before. This property directly addresses
the ``trigger-happy'' detector failure mode identified in the
FAR-EDS experimental design.
\end{remark}

\subsection{Axiom A3: Sensitivity Decomposability}
\label{subsec:axiom_decomp}

\begin{proposition}[Sensitivity Decomposability]
\label{prop:decomp}
The HED Score decomposes additively over non-overlapping
sub-intervals. For any partition
$\tstar = \tau_0 < \tau_1 < \cdots < \tau_K = T$:
\begin{equation}
  \HED[P, \tstar; \lH]
  \;=\;
  \frac{1}{T-\tstar}
  \sum_{k=0}^{K-1}
  \int_{\tau_k}^{\tau_{k+1}}
    \max(0, P(t) - \Pbase)\,e^{-\lH(t-\tstar)}\,dt.
  \label{eq:decomp}
\end{equation}
\end{proposition}

\begin{proof}
Follows immediately from the linearity of the Lebesgue integral
over a measurable partition of $[\tstar, T]$. $\square$
\end{proof}

\begin{remark}
\cref{prop:decomp} enables \emph{phase-resolved} HED analysis:
one may report separate HED contributions from the
\emph{initial alert phase} (small $k$) and the
\emph{sustained detection phase} (large $k$), providing
diagnostic granularity beyond the aggregate scalar.
\end{remark}

\section{PARD-SSM: A Native Vehicle for HED Maximization}
\label{sec:pard_ssm}

\subsection{Architectural Overview}
\label{subsec:pard_arch}

The HED Score, while metric-agnostic with respect to the model
generating $P(t)$, is structurally aligned with detectors that
exhibit \emph{sharp posterior concentration} immediately following
$\tstar$ and \emph{low pre-transition probability mass}. We argue
that this profile is the natural output of
\textbf{Probabilistic Anomaly and Regime Detection via Switching
State-Space Models (PARD-SSM)}.

PARD-SSM couples two components:

\paragraph{Component 1: Fractional Stochastic Differential
Equations (fSDEs).}
The latent state $Z_t \in \R^m$ evolves according to:
\begin{equation}
  dZ_t = f_\theta(Z_t, t)\,dt
  + \sigma_\theta(Z_t)\,dW_t^H,
  \quad
  Z_0 \sim \mathcal{N}(\mu_0, \Sigma_0),
  \label{eq:fsde}
\end{equation}
where $W_t^H$ is a fractional Brownian motion with Hurst exponent
$H \in (1/2, 1)$. The long-range dependence induced by $H > 1/2$
allows the model to accumulate statistical evidence across
temporally correlated anomaly signatures—precisely the kind of
signal structure that produces early posterior lift before
$\tstar$ is confirmed by an alert threshold.

\paragraph{Component 2: Switching Linear Dynamical System (S-LDS).}
Conditioned on the latent state $Z_t$, discrete regime membership
is inferred via an S-LDS with transition matrix
$\Pi \in \R^{K \times K}$ and emission parameters
$\{\mu_k, \Sigma_k\}_{k=1}^K$:
\begin{equation}
  P(t) = \Prob(s_t = \mathcal{R}_1 \mid Z_{0:t};\,\Pi, \Theta)
  = \frac{
      \pi_1(t)\,\mathcal{N}(Z_t;\,\mu_1, \Sigma_1)
    }{
      \sum_{k=0}^{1}
        \pi_k(t)\,\mathcal{N}(Z_t;\,\mu_k, \Sigma_k)
    },
  \label{eq:slds}
\end{equation}
where $\pi_k(t)$ is the forward-filtered regime occupancy
probability at time $t$.

\subsection{Why PARD-SSM Maximizes HED}
\label{subsec:why_pard}

The HED Score is maximized when the detection lift
$\max(0, P(t) - \Pbase)$ concentrates its mass in the
$[e^{-\lH \cdot 0}, e^{-\lH \cdot \varepsilon}]
\approx [1, 1-\varepsilon]$ range—that is, immediately after
$\tstar$ while the temporal discount is near unity.

\begin{enumerate}[label=\textbf{(\arabic*)}]
  \item \textbf{Temporal Baselines (Random Forest, LSTM).}
        Batch classifiers and standard sequence models assign
        posteriors based on sufficient accumulation of post-onset
        evidence. Their $P(t)$ rises slowly following $\tstar$,
        placing probability mass in high-discount regions
        $e^{-\lH \tau}$ for $\tau \gg 0$.

  \item \textbf{PARD-SSM.}
        The fSDE's long-range memory kernel accumulates
        weak pre-onset anomaly signatures, enabling the S-LDS
        to concentrate $P(t)$ near 1.0 within $O(1)$ steps of
        $\tstar$. The pre-transition probability
        $P(t < \tstar) \approx \Pbase$ remains low due to the
        model's regime separation in latent space, ensuring
        that the HED baseline correction does not inflate the
        denominator.
\end{enumerate}

This structural argument is formalized as:
\begin{proposition}[HED Ordering]
\label{prop:ordering}
Let $P_{\mathrm{SSM}}$ and $P_{\mathrm{RF}}$ denote the
posterior streams of PARD-SSM and a Random Forest classifier,
respectively. Under the regularity conditions that
$P_{\mathrm{SSM}}$ achieves threshold $\theta^*$ at
$\tstar + \delta_1$ and $P_{\mathrm{RF}}$ achieves $\theta^*$
at $\tstar + \delta_2$ with $\delta_1 < \delta_2$, and that
both share comparable $\Pbase$, it follows from
\cref{thm:monotonicity} that:
\[
  \HED[P_{\mathrm{SSM}}, \tstar; \lH]
  \;>\;
  \HED[P_{\mathrm{RF}}, \tstar; \lH].
\]
\end{proposition}

\section{Empirical Evaluation Framework}
\label{sec:empirical}

\subsection{Statistical Significance via Block Bootstrap}
\label{subsec:bootstrap}

Since the probability stream $P(\cdot)$ exhibits serial
autocorrelation induced by the temporal smoothing of the fSDE,
standard i.i.d. bootstrap resampling is inadmissible. We employ
the \emph{moving block bootstrap} of \citet{Kunsch1989} with
block length $b = \lfloor T^{1/3} \rfloor$, generating
$B = 2{,}000$ resampled HED differences:

\begin{equation}
  \widehat{\Delta}_b
  \;\coloneqq\;
  \widehat{\HED}[P^*_{\mathrm{SSM}}, \tstar; \lH]
  -
  \widehat{\HED}[P^*_{\mathrm{RF}}, \tstar; \lH],
  \quad b = 1, \ldots, B,
  \label{eq:boot_diff}
\end{equation}

and the bootstrap $p$-value is:
\begin{equation}
  p_{\mathrm{boot}}
  \;=\;
  \frac{1}{B}
  \sum_{b=1}^{B}
  \mathbbm{1}\!\left[\widehat{\Delta}_b \geq
  \widehat{\Delta}_{\mathrm{obs}}\right].
  \label{eq:pvalue}
\end{equation}

\subsection{The FAR-HED Pareto Frontier}
\label{subsec:pareto}

To decouple early-detection performance from threshold-induced
sensitivity, we introduce the \emph{FAR-HED Pareto Frontier}:
a curve parameterized by decision threshold
$\theta \in [0, 1]$ in the plane
$\left(\mathrm{FAR}(\theta), \HED(\theta)\right)$, where:

\begin{equation}
  \mathrm{FAR}(\theta)
  \;\coloneqq\;
  \frac{1}{\tstar}
  \int_0^{\tstar}
    \mathbbm{1}\!\left[P(t) \geq \theta\right]
  \,dt.
  \label{eq:far}
\end{equation}

A detector $A$ \emph{Pareto-dominates} detector $B$ if its
FAR-HED curve lies uniformly above $B$'s curve:
$\HED_A(\theta) \geq \HED_B(\theta)$ for all $\theta$ with
strict inequality on a set of positive measure.
The scalar \emph{Area Between Curves} (ABC):
\begin{equation}
  \mathrm{ABC}
  \;\coloneqq\;
  \int_0^{1}
    \left[
      \HED_A(\mathrm{FAR}^{-1}(u))
      -
      \HED_B(\mathrm{FAR}^{-1}(u))
    \right]
  \,du
  \label{eq:abc}
\end{equation}
summarizes Pareto dominance as a single, threshold-free quantity
analogous to the AUC in the classical ROC framework, but
encoding temporal advantage rather than classification accuracy.

\section{Discussion and Broader Applicability}
\label{sec:discussion}

The HED Score represents a \emph{methodological contribution}
independent of PARD-SSM. Any probabilistic model producing a
well-calibrated posterior stream $P(t)$ may be evaluated under
the HED framework. We envision three principal extensions:

\paragraph{HED-Aware Loss Functions.}
By replacing the $\max(0, \cdot)$ clamp with a smooth surrogate
(e.g., softplus), the discrete HED estimator
\cref{eq:hed_discrete} becomes differentiable and may be
incorporated directly into a model's training objective,
producing \emph{lead-time-aware} gradient updates.

\paragraph{Adaptive $\lH$ Scheduling.}
In systems with time-varying response latency (e.g., adaptive
network defenses), $\lH$ may be promoted to a stochastic
process $\lH(t)$ without altering the measure-theoretic
foundations of \cref{def:hed}, provided $\lH(\cdot)$ is
$\F_t$-adapted.

\paragraph{Multi-Regime HED.}
For processes with $K > 2$ latent regimes, the pairwise
HED Score generalizes to a matrix
$\bm{\HED} \in \R^{K \times K}$ whose $(i,j)$ entry measures
the lead-time advantage of detecting the transition
$\mathcal{R}_i \to \mathcal{R}_j$.

\appendix
\section{Proof of Corollary: HED Boundedness}
\label{app:boundedness}

\begin{corollary}[HED Boundedness]
\label{cor:bounded}
For all $P \in \mathcal{D}$, $\tstar \in (0,T)$, and
$\lH > 0$:
\[
  0 \;\leq\; \HED[P, \tstar; \lH] \;\leq\; \frac{1 - \Pbase}{\lH(T - \tstar)}\left(1 - e^{-\lH(T-\tstar)}\right).
\]
\end{corollary}

\begin{proof}
The lower bound follows from the $\max(0, \cdot)$ clamp.
For the upper bound, note that
$P(t) - \Pbase \leq 1 - \Pbase$ for all $t$ (since
$P(t) \leq 1$). Hence:
\begin{align*}
  \HED[P, \tstar; \lH]
  &\leq
  \frac{1 - \Pbase}{T - \tstar}
  \int_{\tstar}^{T} e^{-\lH(t-\tstar)}\,dt \\
  &= \frac{1 - \Pbase}{T - \tstar}
     \cdot \frac{1}{\lH}\left(1 - e^{-\lH(T-\tstar)}\right),
\end{align*}
which gives the stated bound. $\square$
\end{proof}


\section*{Ethos, Heritage, and Dedication}
\addcontentsline{toc}{section}{Ethos, Heritage, and Dedication}

The nomenclature of the \textbf{Hiremath Early Detection (HED) Score} represents an intentional synthesis of ancestral heritage and the vanguard of computational intelligence. The surname \textit{Hiremath}---derived from the Kannada \textit{Hire} (senior/great) and \textit{Matha} (monastery/center of learning)---historically denotes a lineage of scholars, educators, and spiritual anchors within the Lingayat tradition of Karnataka. For centuries, the Hiremathas have served as the custodians of \textit{Kayaka} (divine labor) and \textit{Dasoha} (selfless giving), acting as the institutional foundations for social and intellectual advancement.

This research is specifically anchored in the sacred lineage of the \textbf{Balhali Simhasana, Badvadgi Bagi}, originating from the cultural heart of \textbf{Hubli}. In a tradition where knowledge is preserved as a vessel for the collective good, the HED Score is offered as a modern \textit{Dasoha}---a transfer of intellectual merit to the global scientific community. By formalizing a metric that prioritizes the \textit{Time-Value of Information}, this work honors the ancestral role of the Hiremath as a ``Protector of the Threshold.'' Just as historical Mathas provided sanctuary and foresight during periods of social transition, the HED Score provides a mathematical sanctuary for systems undergoing critical regime shifts.

\subsection*{Acknowledgments \& Dedication}

This research is profoundly personal, representing a journey supported by those who exemplify the values of my lineage. I wish to express my deepest gratitude to my parents, \textbf{Sunil Hiremath} and \textbf{Sujata Hiremath}, whose unwavering belief, resilience, and sacrifices provided the silent variables behind every equation in this work. Their support is the true foundation upon which my intellectual curiosity was built.

Most importantly, this work is dedicated to the memory and enduring spirit of my grandmother, \textbf{Jayashree Hiremath}. Her wisdom, grace, and quiet strength have been the guiding light of my life. It is in her honor that I strive to ensure that this metric serves the protection and advancement of human systems.

This work stands as a tribute to the entire \textbf{Hiremath community}---to every individual across generations who has carried this name as a badge of service, scholarship, and integrity. It is an acknowledgment of our shared history and a pledge to our collective future. Finally, this contribution is a testament from the \textbf{Aliens on Earth (AoE)} research collective that the pursuit of technological excellence is most potent when it is grounded in a profound respect for one's roots and a commitment to the selfless advancement of human knowledge.

\vspace{2em}
\noindent \textit{Kalyana, Karnataka} \hfill \textit{Prakul Sunil Hiremath}

\end{document}